\pdfoutput=1
\documentclass[conference]{IEEEtran}
\IEEEoverridecommandlockouts
\usepackage{cite}
\usepackage{amsmath,amssymb,amsfonts}
\usepackage{amsthm}
\usepackage{algorithmic}
\usepackage{graphicx}
\usepackage{textcomp}
\usepackage{xcolor}
\usepackage{thmtools}
\usepackage{bm}
\usepackage{todonotes}
\usepackage{dsfont}
\usepackage{xspace}
\usepackage[inline]{enumitem}
\usepackage{etoolbox}
\usepackage{mathtools}
\usepackage{comment}
\usepackage{enumitem}

\declaretheorem[name=Theorem]{thr}

\declaretheorem[name=Example]{example}

\newcommand{\X}[1]{\bm{X}_{#1}}
\newcommand{\Xn}[1]{\bm{X}_{\bar{#1}}}
\newcommand{\Xnn}[2]{\bm{X}_{\bar{#1}_{#2}}}
\newcommand{\x}[1]{\bm{x}_{#1}}
\newcommand{\xn}[1]{\bm{x}_{\bar{#1}}}
\newcommand{\score}[1]{\nu(#1)}
\newcommand{\ymax}{B}
\newcommand{\Var}{\mathbb{V}\mathrm{ar}}
\newcommand{\Cov}{\mathbb{C}\mathrm{ov}}

\newcommand{\E}{\mathbb{E}}

\DeclareRobustCommand{\ie}{i.e.,\@\xspace}

\DeclareMathOperator*{\argmax}{arg\,max}

\DeclareRobustCommand{\quotes}[1]{``#1''}

\AtBeginEnvironment{proof}{\small}
\allowdisplaybreaks[4]

\def\BibTeX{{\rm B\kern-.05em{\sc i\kern-.025em b}\kern-.08em
    T\kern-.1667em\lower.7ex\hbox{E}\kern-.125emX}}

\begin{document}
\setlength{\abovedisplayskip}{2pt}
\setlength{\belowdisplayskip}{2pt}

\title{Feature Selection via Mutual Information:\\New Theoretical Insights}

\author{\IEEEauthorblockN{Mario Beraha\IEEEauthorrefmark{1}\IEEEauthorrefmark{2}, Alberto Maria Metelli\IEEEauthorrefmark{2}, Matteo Papini\IEEEauthorrefmark{2}, Andrea Tirinzoni\IEEEauthorrefmark{2}, Marcello Restelli\IEEEauthorrefmark{2}}
\IEEEauthorblockA{\IEEEauthorrefmark{1} Universit\`a degli Studi di Bologna, Bologna, Italy\\
\IEEEauthorrefmark{2} DEIB, Politecnico di Milano, Milan, Italy\\
Email:\{mario.beraha, albertomaria.metelli, matteo.papini, andrea.tirinzoni, marcello.restelli\}@polimi.it}}


\maketitle

\begin{abstract}
Mutual information has been successfully adopted in filter feature-selection methods to assess both the relevancy of a subset of features in predicting the target variable and the redundancy with respect to other variables. However, existing algorithms are mostly heuristic and do not offer any guarantee on the proposed solution. In this paper, we provide novel theoretical results showing that conditional mutual information naturally arises when bounding the ideal regression/classification errors achieved by different subsets of features. Leveraging on these insights, we propose a novel stopping condition for backward and forward greedy methods which ensures that the ideal prediction error using the selected feature subset remains bounded by a user-specified threshold. We provide numerical simulations to support our theoretical claims and compare to common heuristic methods.
\end{abstract}


\section{Introduction}
The abundance of massive datasets composed of thousands of attributes and the widespread use of learning models able of large representational power pose a significant challenge to machine learning algorithms. Feature selection allows to effectively address some of these challenges with a potential
benefit in terms of computational costs, generalization capabilities and interpretability.
A large variety of approaches has been proposed by the machine learning community~\cite{chandrashekar2014survey}. A simple
dimension for classifying the feature selection methods is whether they are aware of the
underlying learning model. A first group of methods take advantage of this knowledge
and try to identify the best subset of features for the specific model class. This group can be
further split into \emph{wrapper} and \emph{embedded} methods. Wrappers~\cite{kohavi1997wrappers} employ the
learning process as a subroutine of the feature selection process, using the validation error of the
trained model as a score to decide whether to keep or discard a feature. Clearly, this potentially
leads to good generalization capabilities at the cost of iterating the learning process multiple times, which might become impractical for high-dimensional datasets. Embedded methods~\cite{lal2006embedded}, still assume the knowledge of the model class, but the feature selection and the learning process are carried out together (a remarkable example is~\cite{weston2001feature} in which a generalization bound on the SVM is
optimized for both learning the features and the model). Although less demanding than wrappers from a computational standpoint, embedded methods heavily rely on the peculiar properties of
the model class. A second group of methods do not incorporate knowledge of the model
class. These approaches are known as \emph{filters}. Filters~\cite{duch2003feature} perform the feature selection using scores that are independent of the underlying learning model. For this reason, they tend not to overfit but they might result less effective than wrappers and embedded methods as they are general across all the possible model classes. From a computational perspective, filters are the most
efficient feature selection methods.

Filter methods have been deeply studied in the supervised learning field~\cite{duch2006filter}. A relevant amount of literature focused on using the \emph{mutual information} (MI) as a score for identifying a suitable subset of features~\cite{vergara2014review}. The MI~\cite{cover2012elements} is an index of statistical dependence between random variables. Intuitively, the MI measures how much knowing the value of one variable reduces the uncertainty on the other. Differently from other indexes, like the Pearson correlation coefficient, the MI is able to capture also non--linear dependences and is invariant under invertible and differentiable transformations of the random variables~\cite{cover2012elements}. Thanks to these properties, the MI has been employed extensively as a score
for filter methods~\cite{lewis1992feature, battiti1994using, yang2000data, fleuret2004fast, lin2006conditional, cheng2011conditional}. Nonetheless, all these techniques are rather empirical as they try to encode with MI the intuition that \quotes{a feature can be discarded if it is useless for predicting the
target or it is predictable from the other features}. This notion can be made more formal by introducing the notion of \emph{relevance}, \emph{redundancy} and \emph{complementarity}~\cite{vergara2014review}.

To the best of our knowledge, the only work that draws a connection among the several
approaches based on the MI is~\cite{brown2012conditional}. The authors claim that selecting features using as a score the \emph{conditional mutual information} (CMI) is equivalent to maximizing the conditional likelihood between the target and the features. This observation provides a justification to the well--known iterative backward and forward algorithms in which the features are considered one-by-one for insertion in or removal from the feature set, like in the Markov Blanket approach~\cite{tsamardinos2003algorithms}. Although this work offers a wide perspective on the feature selection methods based on the MI, it does
not investigate the relation between the mutual information of a feature set and the prediction error, which, of course, will depend on the specific choice of the model class.

In this paper, we address the problem of controlling the prediction (regression and classification) error when performing the feature selection process via CMI. We claim that selecting features using CMI has the effect on controlling the ideal error, \ie the error attained by the Bayes classifier for classification and the minimum MSE (Mean Squared Error) model for regression. We start in Section~\ref{sec:preliminaries} by revising some fundamental concepts of information theory. In Section~\ref{sec:theory}, we introduce our main theoretical contribution. We derive a pair of inequalities, one for regression (Section~\ref{subsec:regression}) and one for classification (Section~\ref{subsec:classification}), that upper bound the increment of the ideal error obtained by removing a set of features. Such increment is expressed in terms of the CMI between the target and the removed features, given the remaining features. These results support the intuition that a set of features can be safely removed if it does not increase significantly the \quotes{information} about the target, assuming we observed the remaining features. Since the result holds for the \emph{ideal} error, we assert that a filter method based on CMI selects the features assuming that the model employed for solving the regression/classification problem has \quotes{infinite capacity}. We show that, when considering linear models for regression, the bound does not hold and we propose an adaptation for this specific case (Section~\ref{subsec:linear}). These results can be effectively employed to derive a novel and principled stopping condition for the feature selection process (Section~\ref{sec:algo}). Differently from the typical stopping conditions, such as a fixed number of features or the increment of the score, our approach allows to explicitly control the ideal error introduced in the feature selection process. After contextualizing our work in the feature selection literature (Section~\ref{sec:related-works}), we evaluate our approach in comparison with several different stopping criteria on both synthetic and real datasets (Section~\ref{sec:experiments}).

\section{Preliminaries}\label{sec:preliminaries}
We indicate with $\mathcal{X} \subseteq \mathbb{R}^d$ the feature space and with $\mathcal{Y}$ the target space. In case of classification $\mathcal{Y} = \{y_1, y_2, \dots, y_m\}$ is a finite set of classes, whereas in case of regression $\mathcal{Y} \subseteq \mathbb{R}$ is a subset of the real numbers.
 We consider a distribution $p(\bm{X},Y)$ over $\mathcal{X} \times \mathcal{Y}$ from which a finite dataset ${\mathcal{D} = \{ (\bm{x}_i,y_i) | i\in \{1,\dots,N \}\}}$ of $N$ i.i.d. instances is drawn, i.e., $(\bm{x}_i,y_i) \sim p(\bm{X},Y)$ for all $i$. 
For regression problems, we assume there exists $\ymax \in \mathbb R$ such that $|Y| \leq \ymax$ almost surely. A key term for a regression/classification problem is the conditional distribution $p(Y|\bm{x})$, which allows to predict the target associated with any given $\bm{x}\in\mathcal{X}$.

\subsection{Notation}
Given a (random) vector $\bm{X} \in \mathcal{X}$ and a set of indices ${A \subseteq \{1,2,\dots,d\}}$, we denote by $\X{A}$ the vector of components of $\bm{X}$ whose indices are in $A$. Notice that the vectors $\X{A}$ and $\Xn{A}$, for $\bar{A} = \{1,2,\dots,d\} \setminus A$, form a partition of $\bm{X}$.

For a $d$-dimensional random vector $\bm{X}$ we indicate with $\E_{\bm{X}}[\bm{X}]$ the $d$-dimensional vector of the expectations of each component Given two random vectors $\bm{X}$ and $\bm{Y}$, we indicate with $\Cov_{\bm{X},\bm{Y}} \left[ \bm{X}, \bm{Y} \right] = \E_{\bm{X},\bm{Y}} \left[ \left( \bm{X} - \E_{\bm{X}}[\bm{X}]\right) \left( \bm{Y} - \E_{\bm{Y}}[\bm{Y}]\right)^T\right]$ the covariance matrix between the two. We indicate with $\Cov_{\bm{X}}[\bm{X}] = \Cov_{\bm{X}}[\bm{X},\bm{X}]$ the covariance matrix of $\bm{X}$. We denote with $\Var_{\bm{X}}[\bm{X}] = \mathrm{tr}(\Cov_{\bm{X}}[\bm{X},\bm{X}])$ the trace of the covariance matrix of $\bm{X}$. Whenever clear by the context we will remove the subscripts from $\E$, $\Var$ and $\Cov$. Given two random (scalar) random variables $X$ and $Y$ we denote with $\rho(X,Y) = \frac{\Cov[X,Y]}{\sqrt{\Var[X] \Var[Y]}}$ the Pearson correlation coefficient between $X$ and $Y$.

\subsection{Entropy and Mutual Information}
We now introduce the basic concepts from information theory that we employ in the remaining of this paper. For simplicity, we provide the definitions for continuous random variables, although all these concepts straightforwardly generalize to discrete variables~\cite{cover2012elements}.

The \emph{entropy} $H(X)$ of a random variable $X$, having $p$ as probability density function, is a common measure of uncertainty:
\begin{equation}
H(X) \coloneqq \mathbb{E}_X \left[ p(X) \right] = -\int p(x)\log p(x) dx.
\end{equation}
Given two distributions $p$ and $q$, we define the Kullback-Leibler (KL) divergence as:
\begin{equation*}
D_{\mathrm{KL}}(p\|q) \coloneqq \mathbb{E}_X \left[\frac{p(X)}{q(X)} \right] =\int p(x) \log\frac{p(x)}{q(x)}dx.
\end{equation*}
The \emph{mutual information} (MI) between two random variables $X$ and $Y$ is defined as:
\begin{align*}
I(X;Y) &\coloneqq H(Y) - H(Y | X)\\
&=\mathbb{E}_X\left[ D_{\mathrm{KL}}(p(Y|X)\|p(Y))\right]\\
&=\int\int p(x,y) \log \frac{p(x,y)}{p(x) p(y)} \mathrm{d}x\mathrm{d}y,
\end{align*}
Intuitively, the MI between $X$ and $Y$ represents the reduction in the uncertainty of $Y$ after observing $X$ (and viceversa). Notice that the MI is symmetric, i.e., $I(X;Y) = I(Y;X)$.
This definition can be straightforwardly extended by conditioning on a third random variable $Z$, obtaining the \emph{conditional mutual information} (CMI) between $X$ and $Y$ given $Z$:
\begin{align*}
I(X;Y|Z) &\coloneqq \mathbb{E}_Z \left[ I(X|Z;Y|Z) \right]\\
& = \mathbb{E}_Z \left[ \mathbb{E}_X \left[ D_{\mathrm{KL}}(p(Y|X,Z)\|p(Y|Z))\right] \right]\\
& = \int p(z) \int \int p(x,y|z) \log \frac{p(x,y|z)}{p(x|z)p(y|z)} \mathrm{d}x\mathrm{d}y\mathrm{d}z.
\end{align*}
The CMI fulfills the useful chain rule:
\begin{equation}
\label{eq:chain}
	I(X;Y,Z) = I(X;Z) + I(X;Y|Z).
\end{equation}

As we shall see later, the CMI can be used to define a score of both relevancy and redundancy for our feature selection problem, which arises naturally when bounding the ideal regression/classification error. Given a set of indices $A$, we denote the CMI between $Y$ and $\X{A}$ given $\Xn{A}$ as:
\begin{equation*}
\score{A} \coloneqq I(Y;\X{A}|\Xn{A}).
\end{equation*}
This quantity intuitively represents the importance of the feature subset $\X{A}$ in predicting the target $Y$ given that we are also using $\Xn{A}$.

\section{Feature Selection via Mutual Information}\label{sec:theory}

In this section, we introduce our novel theoretical results that shed light on the relationship between CMI and the ideal prediction error. Then, in the next section, we employ these results to propose a new stopping condition that ensures bounded error. We discuss relationships to existing bounds in Section \ref{sec:related-works}.

\subsection{Bounding the Regression Error}
\label{subsec:regression}
We start by analyzing an ideal regression problem under the mean square error (MSE) criterion. Consider the subspace $\mathcal{X}_{\bar{A}}$ of $\mathcal{X}$ which includes only the features with indices in $\bar{A}$ and define $\mathcal{G}_{\bar{A}} = \{g : \mathcal{X}_{\bar{A}} \rightarrow \mathcal{Y} \}$ as the space of all functions mapping $\mathcal{X}_{\bar{A}}$ to $\mathcal{Y}$. The ideal regression problem consists of finding the function $g^* \in \mathcal{G}_{\bar{A}}$ minimizing the expected MSE,
\begin{equation}
\inf_{g \in \mathcal{G}_{\bar{A}}} \mathbb{E}_{\bm{X}, Y} \left[ \left(Y-g(\Xn{A}) \right)^2\right],
\end{equation}
where the expectation is taken under the full distribution $p(\bm{X},Y)$, i.e., under \emph{all} features and the target. The following result relates the ideal error to the expected CMI $\score{A}$.

\begin{thr}\label{th:regression}
Let $\sigma^2 =  \mathbb{E}_{\bm{X}, Y} \left[ \left(Y-\mathbb{E}[Y|\bm{X}]\right)^2\right]$ be the irreducible error and $A$ be a set of indices, then the regression error obtained by removing features $\X{A}$ can be bounded as:
\begin{equation}
	\inf_{g \in \mathcal{G}_{\bar{A}}} \mathbb{E}_{\bm{X}, Y} \left[ \left(Y-g(\Xn{A}) \right)^2\right] \le \sigma^2 + 2\ymax^2 \score{A}.
\end{equation}
\end{thr}

\begin{proof}
The infimum $\inf_{g \in \mathcal{G}_{\bar{A}}} \mathbb{E}_{\bm{X}, Y} \left[ \left(Y-g(\Xn{A}) \right)^2\right]$ is attained by the minimum MSE regression function $g(\bm{x}_{\bar{A}}) = \mathbb{E}_{Y}[Y|\bm{x}_{\bar{A}}]$. Therefore, we have
	\begin{align*}
    	\inf_{g \in \mathcal{G}_{\bar{A}}}\ & \mathbb{E}_{\bm{X}, Y}  \left[ \left(Y-g(\Xn{A}) \right)^2\right] = \mathbb{E}_{ \bm{X}, Y} \left[ \left(Y-\mathbb{E}[Y|\Xn{A}]\right)^2\right] \\
    	& = \int p(\bm{x})\int p(y|\bm{x})\left(
    	y - \mathbb{E}[Y|\xn{A}] \pm \mathbb{E}[Y|\bm{x}]\right)^2 \mathrm{d}y \mathrm{d}\bm{x} \\
        & = \sigma^2 + \int p(\bm{x})\left(  \mathbb{E}[Y|\bm{x}] - \mathbb{E}[Y|\xn{A}] \right)^2 \mathrm{d}\bm{x} \\
        & = \sigma^2 + \int p(\bm{x})\left(  \int y \left( p(y|\bm{x}) - p(y|\xn{A}) \right) \mathrm{d}y \right)^2 \mathrm{d}\bm{x}  \\
        & \le \sigma^2 + \ymax^2 \int p(\bm{x})\left( \int \left| p(y|\bm{x}) - p(y|\xn{A}) \right| \mathrm{d}y \right)^2 \mathrm{d}\bm{x} \\
        & \le \sigma^2 + 2 \ymax^2 \int p(\bm{x}) D_{\mathrm{KL}} \left( p(\cdot|\bm{x}) \| p(\cdot|\xn{A}) \right)\mathrm{d}\bm{x} \\
        & = \sigma^2 + 2 \ymax^2 \score{A}.
    \end{align*}
The second inequality follows from Pinsker's inequality~\cite{pinsker1960information,csiszar1967information,kullback1967lower} by noting that $\int \left| p(y|\bm{x}) - p(y|\xn{A}) \right| \mathrm{d}y = 2 D_{\mathrm{TV}}(p(\cdot | \bm{x})\|p(\cdot | \xn{A}))$ is twice the total variation distance between $p(\cdot | \bm{x})$ and $p(\cdot | \xn{A})$.
\end{proof}

Theorem \ref{th:regression} tells us that the minimum possible MSE that we can achieve by predicting $Y$ only with the feature subset $\bar{A}$ can be bounded by the CMI between $Y$ and $\X{A}$, conditioned on $\Xn{A}$. This result formalizes the intuitive belief that whenever a subset of features $A$ has low relevancy or high redundancy (i.e., $\score{A}$ is small), such features can be safely removed without affecting the resulting prediction error too much. In fact, when $\score{A} = 0$, Theorem \ref{th:regression} proves that it is possible to achieve the irreducible MSE $\sigma^2$ without using any of the features in $A$.

Interestingly, this score accounts for both the relevancy of $\X{A}$ in the prediction of $Y$ and its redundancy with respect to the other features $\Xn{A}$. To better verify this fact, we can rewrite $\score{A}$ as:
\begin{equation}
\int p(\xn{A}) \int p(y,\x{A} | \xn{A})\log \frac{p(y,\x{A} | \xn{A})}{p(y | \xn{A})p(\x{A} | \xn{A})} \mathrm{d}\bm{x}\mathrm{d}y,
\end{equation}
and notice that the inner integral is zero whenever: i) $\xn{A}$ perfectly predicts $y$, i.e., $\x{A}$ is \emph{irrelevant}, or ii) $\xn{A}$ perfectly predicts $\x{A}$, i.e., $\x{A}$ is \emph{redundant}. In both cases we have $p(y,\x{A} | \xn{A}) = p(y | \xn{A})p(\x{A} | \xn{A})$ and, thus, $\score{A} = 0$.

\subsection{Regression Error in Linear Models}
\label{subsec:linear}
As previously mentioned the \emph{actual} error introduced by removing a set of features depends on the choice of the model class. We remark that Theorem \ref{th:regression} bounds the \emph{ideal} prediction error, i.e., the error achieved by a model of infinite capacity. Unfortunately, in practical applications the chosen model has often very limited capacity (e.g., linear). In such cases, our bound, and all CMI-based methods, might be over-optimistic. Indeed, there are situations in which CMI leads to discarding an apparently redundant feature which would reveal itself to be useful when considering the finite capacity of the chosen model. Let us consider the following example.

\begin{example}
\label{example}
Consider a regression problem with two features, $X_1$ and $X_2$, and target $Y = aX_1 + bX_2$, for two scalars $a$ and $b$. Furthermore, assume that $X_1 = Z$ and $X_2 = e^Z$, for $Z \sim \mathcal{N}(0,\sigma^2)$, with $\sigma^2 \gg 0$. It is clear that $\score{\{x_1\}} \simeq 0$ and $\score{\{x_2\}} \simeq 0$, since the two features can be perfectly recovered from one another. However, if the chosen model is linear, both features are fundamental for predicting $Y$. In fact, the squared Pearson correlation coefficients $\rho^2(X_1,Y)$ and $\rho^2(X_2,Y)$ are high, while $\rho^2(X_1,X_2)$ is small.
\end{example}

We show now that, when linear models are involved, the correlation between the features and between a feature and the target can be used to bound the regression error.

\begin{thr}
\label{thr:linear}
Let $\sigma^2_{\bm{X}\rightarrow Y} = \min_{\bm{w},b} \mathbb{E}_{\bm{X}, Y} \left[ \left(Y- \bm{w}^T \bm{X} - b\right)^2\right]$ be the minimum MSE of the linear model that predicts $Y$ with all the features and $(\bm{w}^*, b^*)$ be the optimal weights and bias. Let $A$ be a set of indices and ${\sigma^2_{\Xn{A} \rightarrow X_{i}} = \min_{\bm{w}_{i, \bar{A}}, \bm{b}_{i, \bar{A}}} \mathbb{E}_{X_{i},\Xn{A}} \left[ \left( X_{i} -  \bm{w}_{i, \bar{A}}^T \Xn{A} -{b}_{i, \bar{A}} \right)^2 \right]}$ be the minimum MSE of the linear model that predicts $X_{i}$ from the features
$\Xn{A}$. Then the minimum MSE of the linear model that predicts $Y$ from the features $\Xn{A}$ can be bounded as:
\begin{align*}
	\min_{\bm{w}_{\bar{A}}, b_{\bar{A}}}  \mathbb{E}_{\bm{X}, Y} & \left[ \left(Y- \bm{w}_{\bar{A}}^T \Xn{A} -b_{\bar{A}} \right)^2\right]^{\frac{1}{2}} \\
	& \le \sigma_{\bm{X}\rightarrow Y} + \sqrt{\left| A \right|}\sum_{i \in A} w^*_i \sigma_{\Xn{A} \rightarrow X_{i}}.
\end{align*}
Furthermore, let $\sigma^2_Y = \Var[Y]$. If $\rho(X_i, X_j) = 0$ for all $i,j \in A$ and $i \neq j$ and
 $\rho(X_i, X_j) = 0$ for all $i,j \in \bar{A}$ and $i \neq j$,\footnote{We are requiring that all features in $\X{A}$ are uncorrelated and that all features in $\Xn{A}$ are uncorrelated; but, of course, there might exist $i\in A$ and $j \in \bar{A}$ such that $\rho(X_i, X_j) \neq 0$.} then it holds that:
\begin{align*}
	\min_{\bm{w}_{\bar{A}}, b_{\bar{A}}} & \mathbb{E}_{\bm{X}, Y} \left[ \left(Y- \bm{w}_{\bar{A}}^T \Xn{A} -b_{\bar{A}} \right)^2\right]^{\frac{1}{2}}  \le  \sigma_{\bm{X}\rightarrow Y} \\
	& + \sqrt{\left| A \right|} \sigma_Y \sum_{i \in A} \rho(Y,X_i) \bigg(1 - \sum_{j \in \bar{A}} \rho(X_i, X_j)^2 \bigg)^{\frac{1}{2}}.
\end{align*}
\end{thr}

\begin{proof}
{\small
	Consider the linear regression problem for predicting $Y$ with all the features, $\min_{\bm{w}, b} \mathbb{E}_{Y, \bm{X}} \left[ \left( Y - \bm{w}^T \bm{X} - b \right)^2 \right]$, having $(\bm{w}^*, b^*)$ as the optimal solution. The expression of the optimal weights and the minimum MSE $\sigma^2$ are given by:
\begin{align*}
	\bm{w}^* = \Cov[\bm{X}]^{-1}\Cov[\bm{X},Y].
\end{align*}
\begin{equation*}
	\sigma^2_{\bm{X}\rightarrow Y} = \Var[Y] - \Cov[Y,\bm{X}] \Cov[\bm{X}]^{-1} \Cov[\bm{X},Y].
\end{equation*}
Consider now a partition of $\bm{X}$ into $\Xn{A}$ and $\X{A}$ and the linear regression problem to predict $\X{A}$ from $\Xn{A}$, \ie $\min_{\bm{W}_{A, \bar{A}}, \bm{b}_{A, \bar{A}}} \mathbb{E}_{\X{A},\Xn{A}} \left[ \left( \X{A} -  \bm{W}_{A, \bar{A}} \Xn{A} -\bm{b}_{A, \bar{A}} \right)^2 \right]$. We can express the optimal weights and the minimum MSE as:
\begin{align*}
	 \bm{W}_{A, \bar{A}}^* = \Cov[\X{A},\Xn{A}] \Cov[\Xn{A}]^{-1}.
\end{align*}
\begin{equation*}
	\sigma^2_{\Xn{A} \rightarrow X_{i}} = \Var[{X}_i] - \Cov[{X}_i,\Xn{A}] \Cov[\Xn{A}]^{-1} \Cov[\Xn{A},{X}_i]
\end{equation*}
Let us now consider the linear regression problem for predicting $Y$ from the features $\Xn{A}$.
	\begin{align}
		& \min_{\bm{w}_{\bar{A}}, b_{\bar{A}}} \mathbb{E}_{\bm{X},Y} \left[ \left( Y -  \bm{w}_{\bar{A}}^T \Xn{A} -b_{\bar{A}} \right)^2 \right]^{\frac{1}{2}}  \notag \\
		& \le \mathbb{E}_{\bm{X},Y} \left[ \left( Y -  {\bm{w}_{\bar{A}}^*}^T \Xn{A} -b^* \pm {\bm{w}_{A}^*}^T \left(\bm{W}_{A, \bar{A}}^* \Xn{A} -\bm{b}_{A, \bar{A}}^* \right) \right)^2 \right]^{\frac{1}{2}} \notag \\
		& \le \mathbb{E}_{\bm{X},Y} \left[ \left( Y -  {\bm{w}_{\bar{A}}^*}^T \Xn{A} -{\bm{w}_{{A}}^*}^T \X{A} -b^*\right)^2 \right]^{\frac{1}{2}} \notag \\
		& \quad + \mathbb{E}_{\bm{X},Y} \left[\left( {\bm{w}_{{A}}^*}^T \left( \X{A} - \bm{W}_{A, \bar{A}}^* \Xn{A} -\bm{b}_{A, \bar{A}}^* \right) \right)^2 \right]^{\frac{1}{2}} \label{p:1}\\
		& \le \sigma_{\bm{X}\rightarrow Y} + \sqrt{\left| A \right|} \mathbb{E}_{\bm{X}} \left[ \sum_{i \in A} {{w}_{i}^*}^2 \left( X_{i} - {\bm{w}_{i, \bar{A}}^*}^T \Xn{A} -{b}_{i, \bar{A}}^* \right)^2  \right]^{\frac{1}{2}} \label{p:2}\\
		& = \sigma_{\bm{X}\rightarrow Y} + \sqrt{\left| A \right|} \left( \sum_{i \in A} {w_{i}^*}^2  \mathbb{E}_{\bm{X}} \left[ \left( X_{i} - {\bm{w}_{i, \bar{A}}^*}^T \Xn{A} - {b}_{i, \bar{A}}^* \right)^2  \right] \right)^{\frac{1}{2}} \notag \\
		& \le\sigma_{\bm{X}\rightarrow Y} + \sqrt{\left| A \right|} \sum_{i \in A}   {{w}_{i}^*}  \mathbb{E}_{\bm{X}} \left[ \left( X_{i} - {\bm{w}_{i, \bar{A}}^*}^T \Xn{A} -{b}_{i, \bar{A}}^* \right)^2  \right]^{\frac{1}{2}} \label{p:3}\\
		& = \sigma_{\bm{X}\rightarrow Y} + \sqrt{\left| A \right|} \sum_{i \in A}   {w_{i}^*}  \sigma_{\Xn{A} \rightarrow X_{i}}, \notag
	\end{align}
	where \eqref{p:1} derives from Minkowski inequality after having summed and subtracted ${\bm{w}_{{A}}^*}^T \X{A}$, \eqref{p:2} is obtained from Cauchy-Schwarz inequality (for $d$ dimensional vectors we have $(\bm{a}^T\bm{b})^2 \le d \sum_{i=1}^d a_i^2 b_i^2$) and \eqref{p:3} derives from subadditivity of the square root.

By recalling that $ {w_{i}^*} = \sum_{j \in A} \Cov[X_i,X_j]^{-1} \Cov[X_j,Y]$, for uncorrelated features we get:
\begin{align*}
	{w_{i}^*} & = \Var[X_i]^{-1} \Cov[X_i,Y] = \left(\frac{\Var[Y]}{\Var[X_i]} \right)^{\frac{1}{2}}\rho(Y,X_i).
\end{align*} If the features in $\Xn{A}$ are uncorrelated as well, we have that $\Cov[\Xn{A}]$ is diagonal. Therefore, we have:
	\begin{align*}
	\sigma^2_{\Xn{A} \rightarrow X_{i}} & = \Var[{X}_i] - \sum_{j \in \bar{A}} \Cov[{X}_i, X_j]^2 \Var[X_j]^{-1} \\
	& = \Var[{X}_i] \bigg(1 - \sum_{j \in \bar{A}} \rho(X_i, X_j)^2 \bigg),
	\end{align*}
	from which the result follows directly.
}
\end{proof}

This result allows highlighting two relevant points. First, when considering linear models what matters is the correlation among the features and the correlation between the features and the class. Most importantly, the Pearson correlation coefficient is a weaker index of dependency between random variables compared to the MI as it identifies linear dependency only. As suggested by Example~\ref{example}, using MI for discarding features when the model used for prediction is too weak might be dangerous. Second, Theorem~\ref{thr:linear} highlights once again two relevant properties of the features. In the linear case, a feature $X_i$ is relevant if it is highly correlated with the target $Y$, \ie $\rho^2(Y,X_i) \gg0$, and a feature is redundant if it is highly correlated with the others, \ie $\rho^2(X_i, X_j) \gg 0$. Both these contributions appear clearly in Theorem~\ref{thr:linear}.

\subsection{Bounding the Classification Error}
\label{subsec:classification}
A similar result to Theorem \ref{th:regression} can be obtained for an ideal classification problem. Here the goal is to find the function $g^* \in \mathcal{G}_{\bar{A}}$ minimizing the ideal prediction loss,
\begin{equation}
\inf_{g \in \mathcal{G}_{\bar{A}}} \mathbb{E}_{\bm{X}, Y} \left[ \mathds{1}_{\left\{ Y \neq g(\Xn{A}) \right\}} \right],
\end{equation}
where $\mathds{1}_{E}$ denotes the indicator function over an event $E$.

\begin{thr} \label{th:classification}
Let $\epsilon = \mathbb{E}_{\bm{X}, Y} \left[ \mathds{1}_{\left\{ Y \neq \argmax_{y \in \mathcal{Y}} p(y | \bm{X}) \right\}} \right]$ be the Bayes error and $A$ be a set of indices, then the classification error obtained by removing features $\X{A}$ can be bounded as:
\begin{equation}
	\inf_{g \in \mathcal{G}_{\bar{A}}} \mathbb{E}_{\bm{X}, Y} \left[ \mathds{1}_{\left\{ Y \neq g(\Xn{A}) \right\}} \right] \le \epsilon + \sqrt{2 \score{A} }.
\end{equation}
\end{thr}

\begin{proof}
Let us denote by $y^* = \argmax_{y \in \mathcal{Y}} p(y | \bm{x})$ the optimal prediction given $\bm{x}$ and by $y^*_{\bar{A}} = \argmax_{y \in \mathcal{Y}} p(y | \xn{A})$ the optimal prediction given the subset of features in $\bar{A}$. We have:
	\begin{align*}
    	&\inf_{g \in \mathcal{G}_{\bar{A}}}\  \mathbb{E}_{\bm{X}, Y}  \left[ \mathds{1}_{\left\{ Y \neq g(\Xn{A}) \right\}} \right] = \mathbb{E}_{\bm{X}, Y}  \left[ \mathds{1}_{\left\{ Y \neq \argmax_{y \in \mathcal{Y}} p(y | \Xn{A}) \right\}} \right] \\
    	& = \mathbb{E}_{\bm{X}, Y}  \left[ \mathds{1}_{\left\{ Y \neq \argmax_{y \in \mathcal{Y}} p(y | \Xn{A}) \right\}} \pm \mathds{1}_{\left\{ Y \neq \argmax_{y \in \mathcal{Y}} p(y | \bm{X}) \right\}} \right] \\
        & = \epsilon + \int p(\bm{x}) \left( p(y^*| \bm{x}) -p(y^*_{\bar{A}}| \bm{x})  \right) \mathrm{d}\bm{x} \\
        & = \epsilon + \int p(\bm{x}) \left( p(y^*| \bm{x}) \pm p(y^*_{\bar{A}}| \xn{A}) -p(y^*_{\bar{A}}| \bm{x})  \right) \mathrm{d}\bm{x}.
    \end{align*}
Let us now bound the term inside the expectation point-wisely. For the term $p(y^*| \bm{x}) - p(y^*_{\bar{A}}| \xn{A})$, we have:
\begin{align*}
p(y^*| \bm{x}) - p(y^*_{\bar{A}}| \xn{A}) =& \max_{y \in \mathcal{Y}} p(y| \bm{x}) - \max_{y \in \mathcal{Y}}p(y| \xn{A})\\ \leq& \max_{y \in \mathcal{Y}} | p(y| \bm{x}) - p(y| \xn{A}) |\\ \leq&  D_{\mathrm{TV}}(p(\cdot | \bm{x})\|p(\cdot | \xn{A})).
\end{align*}
Following a similar argument for the term $p(y^*_{\bar{A}}| \xn{A}) -p(y^*_{\bar{A}}| \bm{x})$, we find that the inner term is always less or equal than the total variation distance between $p(\cdot | \bm{x})$ and $p(\cdot | \xn{A})$. Then, by applying Pinsker's inequality:
\begin{align*}
\inf_{g \in \mathcal{G}_{\bar{A}}}\ \mathbb{E}_{\bm{X}, Y} & \left[ \mathds{1}_{\left\{ Y \neq g(\Xn{A}) \right\}} \right] \leq \epsilon + 2 \int p(\bm{x}) D_{\mathrm{TV}}(p(\cdot | \bm{x})\|p(\cdot | \xn{A})) \mathrm{d}\bm{x}\\ & \le \epsilon + \int p(\bm{x}) \sqrt{2 D_{\mathrm{KL}} \left( p(\cdot | \bm{x}) \| p(\cdot | \xn{A}) \right) }\mathrm{d}\bm{x} \\
        & \le \epsilon + \sqrt{2  \int p(\bm{x}) D_{\mathrm{KL}} \left( p(\cdot | \bm{x}) \| p(\cdot | \xn{A}) \right) \mathrm{d}\bm{x} } \\
        & = \epsilon + \sqrt{2 \score{A} }.
\end{align*}
Here the last inequality follows from Jensen's inequality and the concavity of the square root.
\end{proof}

Similarly to the result for regression problems, Theorem \ref{th:classification} bounds the minimum ideal classification error achievable by a model which uses only the subset of features in $\bar{A}$ by the score $\score{A}$. The astute reader might have noticed a slightly better dependence on $\score{A}$ with respect to the regression case (square root versus linear). This is due to the fact that minimizing the MSE gives rise to a squared total variation distance between the conditional distributions $p(\cdot | \bm{x})$ and $p(\cdot | \xn{A})$, which leads to a linear dependence on $\score{A}$.


\section{Algorithms}\label{sec:algo}
In this section, we rephrase the forward and backward feature selection algorithms based on the findings of Section \ref{sec:theory}. Furthermore, we propose a novel stopping condition
that allows to bound the error introduced by removing a set of features, assuming the predictor will make the best possible use of the remaining features. Actively searching for the optimal subset of features is combinatorial in the number of features and, thus, unfeasible~\cite{john1994irrelevant}. Instead, we can start from the complete feature set and remove one feature at a time, greedily minimizing the score. In this spirit, we propose the following iterative procedure.
%
%
%
\begin{restatable}[Backward Elimination]{algo}{be}\label{algo:be}
	Given a dataset $\bm{X}, Y$, select a threshold $\delta\geq0$, the maximum error that the filter is allowed to introduce. Then:
	\begin{itemize}
		\item Start with the full feature set, \ie $A_1 = \emptyset$, where $A_t$ denotes the index set of features removed prior to step $t$.
		\item For each step $t=1,2\dots$, remove the feature that minimizes the conditional mutual information between itself and the target $Y$ given the remaining features, i.e.:
		\begin{align}
		&i_t =\arg\min_{i}I(Y;X_{i} | \Xn{A_t}\setminus X_{i}),\\
		&I_t = I(Y;X_{i_t} | \Xn{A_t}\setminus X_{i_t}),\label{eq:it}\\
		&A_{t+1} = A_t \cup \{i_t\} \label{eq:bes}
		\end{align}
		\item Stop as soon as $\sum_{h=1}^{t}I_h \geq \frac{\delta}{2\ymax^2}$ for regression and $\sum_{h=1}^{t}I_h \geq \frac{\delta^2}{2}$ for classification. The selected features are the remaining ones, indexed by $\overline{A}_T$, where $T$ is the last step.
	\end{itemize}
\end{restatable}
This algorithm, apart from the stopping condition, is described by Brown et al.~\cite{brown2012conditional} as \textit{Backward Elimination with Mutual Information}. The same authors show that this procedure greedily maximizes the conditional likelihood of the selected features given the target, as long as $I_k$ is always zero. This would correspond to selecting $\delta=0$ as a threshold in our algorithm. The same backward elimination step is used as a subroutine in the IAMB algorithm~\cite{tsamardinos2003algorithms}. Our stopping condition allows selecting the maximum error that the feature selection procedure is allowed to add to the ideal error, \ie the unavoidable error that even a perfect predictor using all the features would commit. The fact that the threshold will be actually observed is guaranteed by the following result.
\begin{restatable}{thr}{beth}\label{th:beth}
	Algorithm \ref{algo:be} achieves an error of $\sigma^2 + \delta $ for regression, where $\sigma^2$ is the irreducible error and $\epsilon + \delta$ for classification, where $\epsilon$ is the Bayes error.
\end{restatable}
\begin{proof}
	We prove the result for regression using Theorem \ref{th:regression}. The proof for classification is analogous, but based on Theorem \ref{th:classification}. We have:
	\begin{align}
		\inf_{g\in \mathcal{G}_{\bar{A}}} \mathbb{E}_{\bm{X}, Y} \left[ \left(Y-g(\Xn{A_t}) \right)^2\right] \le \sigma^2 + 2\ymax^2 \score{A_t},\label{eq:4.1.1}
	\end{align}
	where $t$ is any iteration of the algorithm. By repeatedly applying the chain rule of CMI \eqref{eq:chain}, we can rewrite the score as:
	\begin{align}
		\score{A_{t+1}}
		&= I(Y;\X{A_{t+1}} | \Xnn{A}{t+1}) \nonumber\\
		&= I(Y; \bm{X}) - I(Y; \Xnn{A}{t+1}) \nonumber\\
		&= I(Y; \X{A_t},\Xn{A_t}) - I(Y; \Xnn{A}{t+1}) \nonumber\\
		&= I(Y; \X{A_t} | \Xn{A_t}) + I(Y; \Xn{A_t}) - I(Y; \Xnn{A}{t+1}) \nonumber\\
		&= \score{A_t} + I(Y; X_{i_t},\Xnn{A}{t+1}) - I(Y; \Xnn{A}{t+1}) \nonumber\\
		&= \score{A_t} + I(Y; X_{i_t} | \Xnn{A}{t+1}) \pm I(Y; \Xnn{A}{t+1}) \nonumber\\
		&= \score{A_t} + I(Y; X_{i_t} | \Xn{A_{t}}\setminus X_{i_t}) \nonumber\\
		&= \score{A_t} + I_t.\label{eq:4.1.2}
	\end{align}
	Noting that $\nu(A_1) = I(Y; \emptyset | \bm{X}) = 0$, we can unroll this recursive equation, obtaining:
	\begin{align}\label{eq:4.1.3}
		\nu(A_T) = \sum_{t=1}^{T-1}I_t \leq \frac{\delta}{2\ymax^2},
	\end{align}
	where the inequality is due to the stopping condition. Plugging (\ref{eq:4.1.3}) into (\ref{eq:4.1.1}), we get the thesis.

\end{proof}
%
%

%

Our Theorems~\ref{th:regression} and~\ref{th:classification} suggest that a backward elimination procedure allows keeping the
error controlled. In the following, we argue that we can resort also to forward selection methods and still have a guarantee
on the error. Using the chain rule of the CMI we can express our score $\score{A}$ as:
\begin{equation*}
	\score{A} = I(Y;\bm{X}) - I(Y;\Xn{A}),
\end{equation*}
where $\Xn{A}$ is the set of features that have not been eliminated yet. If we plug this equation into the bounds of Theorems~\ref{th:regression} and~\ref{th:classification} we get:
\vspace{2mm}\\
\resizebox{.5\textwidth}{!}{
$\displaystyle
	\inf_{g \in \mathcal{G}_{\bar{A}}} \mathbb{E}_{\bm{X}, Y} \left[ \left(Y-g(\Xn{A}) \right)^2\right] \le \sigma^2 + 2\ymax^2 \left[ I(Y;\bm{X}) - I(Y;\Xn{A}) \right],$
}
\resizebox{.5\textwidth}{!}{
$\displaystyle
	\inf_{g \in \mathcal{G}_{\bar{A}}} \mathbb{E}_{\bm{X}, Y} \left[ \mathds{1}_{\left\{ Y \neq g(\Xn{A}) \right\}} \right] \le \epsilon + \sqrt{2 \left[I(Y;\bm{X}) - I(Y;\Xn{A})\right] },$
}
\vspace{0.5mm}\\
for the regression and classification cases respectively.
Since $I(Y;\bm{X})$ does not depend on the selected features $\Xn{A}$, in order to minimize the bound we need to maximize the term $I(Y;\Xn{A})$. This matches the intuition that we should select the features that provide the maximum information on the class. Using this result, we can easly
provide a forward feature selection algorithm.

\begin{restatable}[Forward Selection]{algo}{be}\label{algo:fs}
	Given a dataset $\bm{X}, Y$, select a threshold $\delta\geq0$, the maximum error that the filter is allowed to introduce. Then:
	\begin{itemize}
		\item Start with the empty feature set, \ie $A_1 = \emptyset$, where $A_t$ denotes the index set of features selected prior to step $t$.
		\item For each step $t=1,2\dots$, add the feature that maximizes the conditional mutual information between itself and the target $Y$ given the remaining features, i.e.:
		\begin{align}
		&i_t =\arg\max_{i}I(Y;X_{i} | \X{A_{t}}),\\
		&I_t = I(Y;X_{i_t} | \X{A_t}),\label{eq:it}\\
		&A_{t+1} = A_t \cup \{i_t\} \label{eq:bes}
		\end{align}
		\item Stop as soon as $\sum_{h=1}^{t}I_h \ge \frac{\delta}{2\ymax^2}$ for regression and $\sum_{h=1}^{t}I_h \ge \frac{\delta^2}{2}$ for classification. The selected features are those indexed by $A_T$, where $T$ is the last step.
	\end{itemize}
\end{restatable}

Apart from the stopping condition, this algorithm was also presented in Brown et al.~\cite{brown2012conditional} and named \textit{Forward Selection with Mutual Information}. Like for the backward case, we are able to provide a guarantee on the final error.

\begin{restatable}{thr}{beth}\label{th:fs}
Algorithm \ref{algo:fs} achieves an error of $\sigma^2 - \delta + 2B^2 I(Y;\bm{X}) $ for regression, where $\sigma^2$ is the irreducible error and $\epsilon - \delta + \sqrt{2I(Y;\bm{X})}$ for classification, where $\epsilon$ is the Bayes error.
\end{restatable}
\begin{proof}
	We prove the result just for the regression case, as the derivation for classification is analogous. Using the chain rule~\eqref{eq:chain}, we have the following recursion:
	\begin{align*}
		I(Y; \bm{X}_{A_{t+1}}) & = I(Y; \bm{X}_{A_{t}} , X_{i_t} ) \\
		& = I(Y; X_{i_t} | \bm{X}_{A_{t}} ) + I(Y; \bm{X}_{A_{t}}) \\
		& = I_t + I(Y; \bm{X}_{A_{t}}).
	\end{align*}
	By observing that $I(Y;\X{A_1}) = I(Y; \emptyset) = 0$, we unroll the recursion and we get
	\begin{equation*}
		I(Y; \bm{X}_{A_{T}}) = \sum_{t=1}^T I_t \ge \frac{\delta}{2\ymax^2},
	\end{equation*}
	from which the result follows.
\end{proof}

\subsection{Estimation of the Conditional Mutual Information}
So far, we have assumed to be able to compute the CMI $I(Y; X_{i} | \Xn{A_t} \setminus X_{i})$ and $I(Y;X_{i} | \X{A_t})$ exactly. In practice, they need to be estimated from data. Estimating the MI can be reduced to the estimation of several entropies~\cite{paninski2003estimation}; numerous methods have been employed in feature selection, either based on nearest neighbors approaches~\cite{tsimpiris2012nearest} or on histograms~\cite{brown2012conditional}. The main challenge arise in classification where we need to estimate CMI between a discrete variable (the class) and possibly continuous features. For this reason, we resort to the recent nearest neighbor estimator proposed by~\cite{gao2017estimating}, which collapses to the more traditional KSG estimator~\cite{kraskov2004ksg} when both $X$ and $Y$ have a continuous density. These estimators are proved to be consistent when the number of samples and the number of neighbors grows to infinity~\cite{gao2017estimating}.

\section{Related Works}\label{sec:related-works}

A related theoretical study of feature selection via MI has been recently proposed by Brown et al. \cite{brown2012conditional}. The authors show that the problem of finding the minimal feature subset such that the conditional likelihood of the targets is maximized is equivalent to minimizing the CMI. Based on this result, common heuristics for information-theoretic feature selection can be seen as iteratively maximizing the conditional likelihood. Similarly, we show a connection between the CMI and the optimal prediction error. Differently from \cite{brown2012conditional}, we additionally propose a novel stopping condition that is well motivated by our theoretical findings.

In the information theory literature, \cite{wu2012functional} also analyzes the connection between CMI and minimum mean square error, deriving a similar result to our Theorem \ref{th:regression}. However, classification problems (i.e., minimum zero-one loss) are not considered and the focus is not on feature selection.

The authors of \cite{tsimpiris2012nearest} propose a nearest neighbor estimator for the CMI and show how it can be used in a classic forward feature selection algorithm. One of the authors' questions is how to devise a suitable stopping condition for such methods. Here we propose a possible answer: our stopping criterion (Section \ref{sec:algo}) is intuitive, applicable to both forward and backward algorithms, and theoretically well-grounded.

Several existing approaches use linear correlation measures to score the different features \cite{das1971feature,hall1999correlation,yu2003feature,biesiada2007feature,eid2013linear}.
Such algorithms are mostly based on the heuristic intuition that a good feature should be highly correlated with the class and lowly correlated with the other features. Instead, we provide a more theoretical justification for this claim (Section \ref{sec:theory}), showing a connection between these two properties and the minimum MSE.

\section{Experiments}\label{sec:experiments}
We evaluate the performance of our stopping condition on synthetic and real-world datasets, by comparing different stopping criteria, employing a backward feature selection approach:
\begin{itemize}
\item \textit{error (ER)}: stop when the bound on the prediction error, as in Theorem \ref{th:beth}, is greater than a fixed threshold $\delta$;
\item \textit{feature score (FS)}: stop when a feature with a CMI score greater than a fixed threshold $\delta$ is encountered;
\item \textit{delta feature score ($\Delta$FS)}: stop when the difference between the score of two consecutive features is greater than a threshold $\delta$ (as in knee-elbow analysis);
\item \textit{number of features (\#F)}: stop with exactly $k$ features are selected.
\end{itemize}
For all the experiments, we use Python's scikit-learn implementation of SVM with default parameters (RBF kernel and $C=1$).

\subsection{Synthetic Data}
The synthetic data consist in several binary classification problems. Each dataset is composed of $500$ samples. The datasets are generated, similarly to \cite{chen2017kernel}, as follows: fix the number of useful features $k$ (\ie the number of features that are actually needed to predict the class); given $Y=1$, $X_1, \dots X_k$ are $\mathcal{N}(0, 1)$ conditioned on $\sum_{i=1}^k X_i > 3(k-2)$, while $X_{k+1}, \dots X_{15} \sim  \mathcal{N}(0, 1)$. The choice of $k$ will be specified for each experiment.

\begin{figure*}[t]
\includegraphics[width=\textwidth]{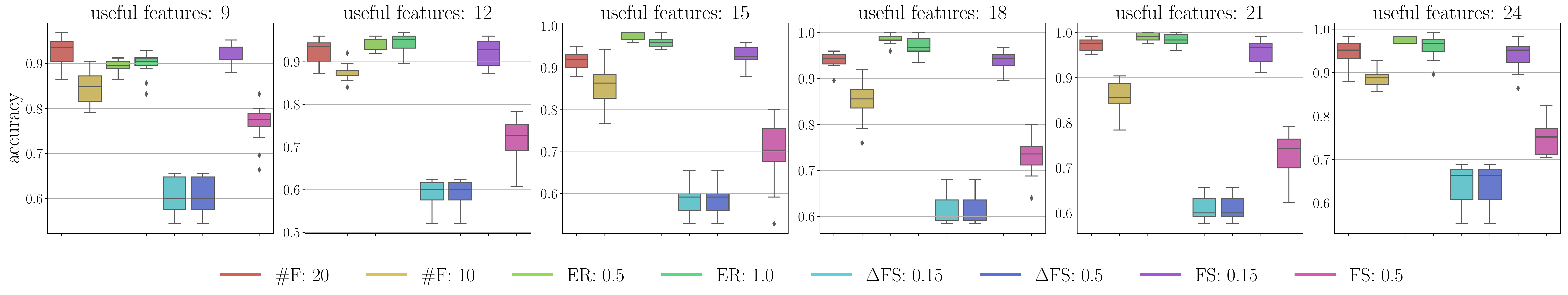}
\caption{SVM test accuracy for different choices of the number of features that generate the problem and different stopping criteria.}
\label{fig:accuracy_by_features}
\end{figure*}

\textit{Stopping Condition Comparison. }
The first experiment is meant to compare the stopping conditions presented above across datasets for classification with different a number of useful features. We generate $6$ independent problems with $30$ features. Among the $30$ features only $k \in \{ 9, 12, 15, 18, 21, 24 \}$ are useful to predict the target.
In Figure \ref{fig:accuracy_by_features}, we show the accuracy of SVM
for the different datasets and different stopping conditions. We can see that our stopping condition (ER) performs better than choosing a fixed number of features (\#F) in most cases. More notably, the feature selection algorithm shows a greater robustness w.r.t the stopping condition's hyper parameter with our error-based criterion, as one would expect.
Furthermore, in Figure \ref{fig:accuracy_by_features}, we notice that the delta feature score ($\Delta$FS) (which is similar to the knee-elbow analysis) is highly inefficient (as the outputs are almost identical for both choices of the threshold) and is clearly the worst performer.
The feature score (FS) stopping criterion is highly sensitive to the threshold, achieving good performance with a low threshold and a significantly worse performance when the threshold is increased.
The choice of the threshold in both $\Delta$FS and FS poses a significant problem as it has no relation to the prediction error and its optimal value is highly problem-specific. On the contrary, for \#F and ER criteria the hyper parameter has a precise meaning and thus it can be selected more easily.

\textit{Robustness. }
To have a better grasp of the proposed stopping criterion we generate $50$ binary classification problems, with the only difference of choosing $k$ accordingly to  ${k \sim \text{Uniform}(3, 15)}$ and having only $15$ total features.
In Figure \ref{fig:accuracy_and_selected} we show the accuracy of a SVM classifier
on a test set, after the feature selection has been performed, as a function of the error threshold $\delta$. Moreover, we overlay the fraction of selected features over the original $15$. We can notice two interesting facts.
\begin{enumerate*}[label=\roman*)]
\item Even with a threshold close to zero\footnote{Since the CMI is estimated from data as well, we cannot set the threshold to exactly $0$, thus, we used $0.05$ in the experiments.} a great number of features is discarded.
\item The classification accuracy is rather constant despite a high error threshold while the number of selected features decreases significantly.
\end{enumerate*}
We can conclude that our method was effectively able to identify irrelevant features and discard them.

\textit{CMI Estimation. } To see how the estimation of the (conditional) mutual information impacts on the performance of the stopping condition, we consider one last problem, generated as before with $N=30$ features and fixed $k=10$. In Figure \ref{fig:binary_big}, we look at the performance of an SVM classifier on the same test set for increasing sizes (number of samples) of the training set. We select the number of neighbors in the mutual information estimation as a fixed fraction of the training set size.
Notice how, when the data points are too few, the estimated mutual information \quotes{overfits}, and actually very little to no features are discarded in the feature selection step. As a consequence, also the SVM classifier overfits the training set and leads to poor performance on the test set.
On the other hand, as the number of samples increases, the estimation of the mutual information becomes more precise and the appropriate set of features is selected, resulting in a great increase in the classification accuracy on the test set.
Moreover, for a small number of data points, the number of neighbors used in the MI estimation is not too relevant, while it is evident that for a large enough sample size, it is better to increase the number of neighbors.

\begin{figure*}[t] 
  \begin{minipage}[b]{0.5\linewidth}
    \centering
    \includegraphics[width=\linewidth]{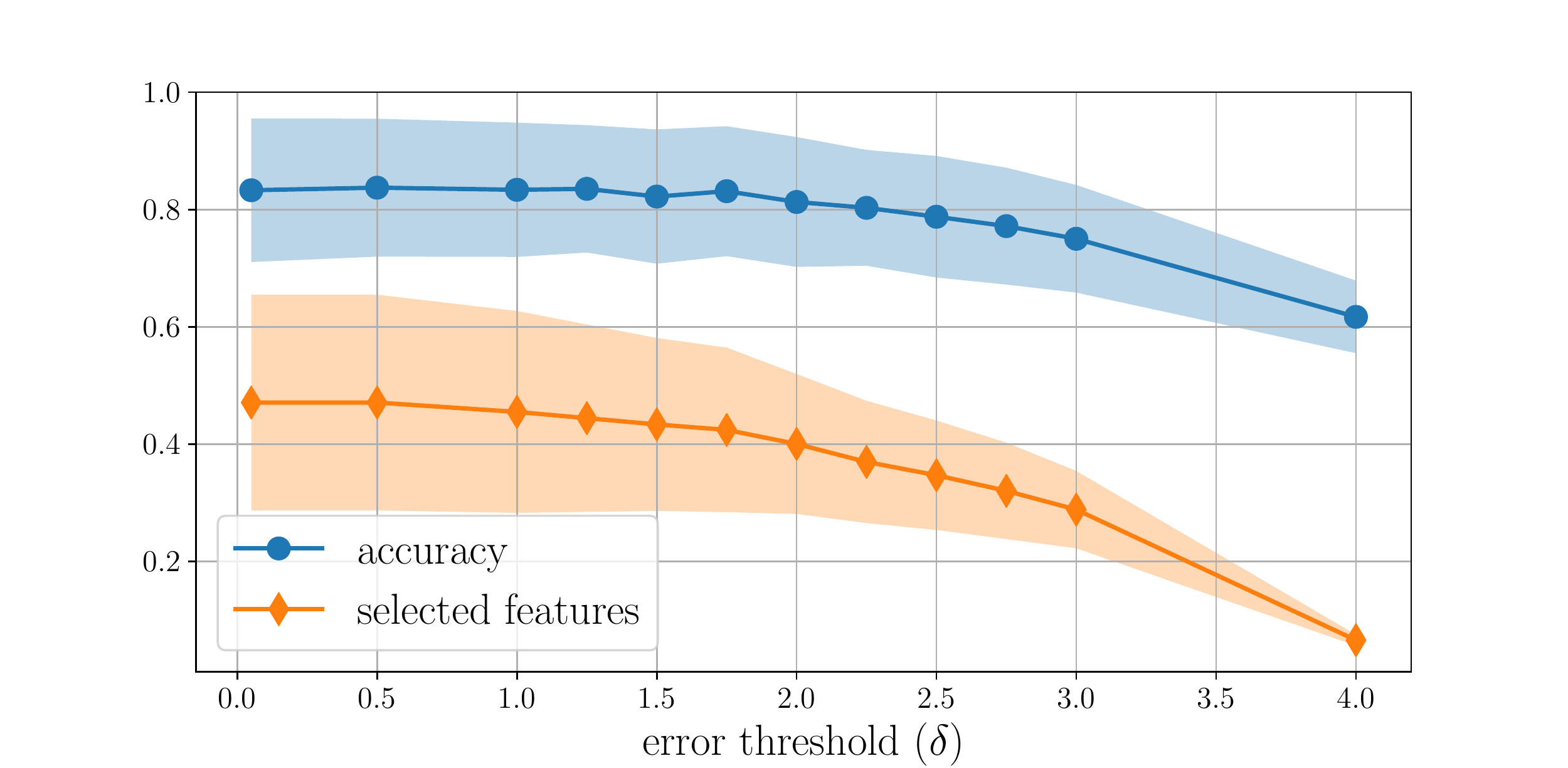}
\caption{Classification accuracy and fraction of selected features as a function of the error threshold $\delta$. Estimates are reported as mean values $\pm$ standard deviation.}
\label{fig:accuracy_and_selected}
  \end{minipage}
  \hspace{0.5cm}
  \begin{minipage}[b]{0.5\linewidth}
    \centering
    \includegraphics[width=\linewidth]{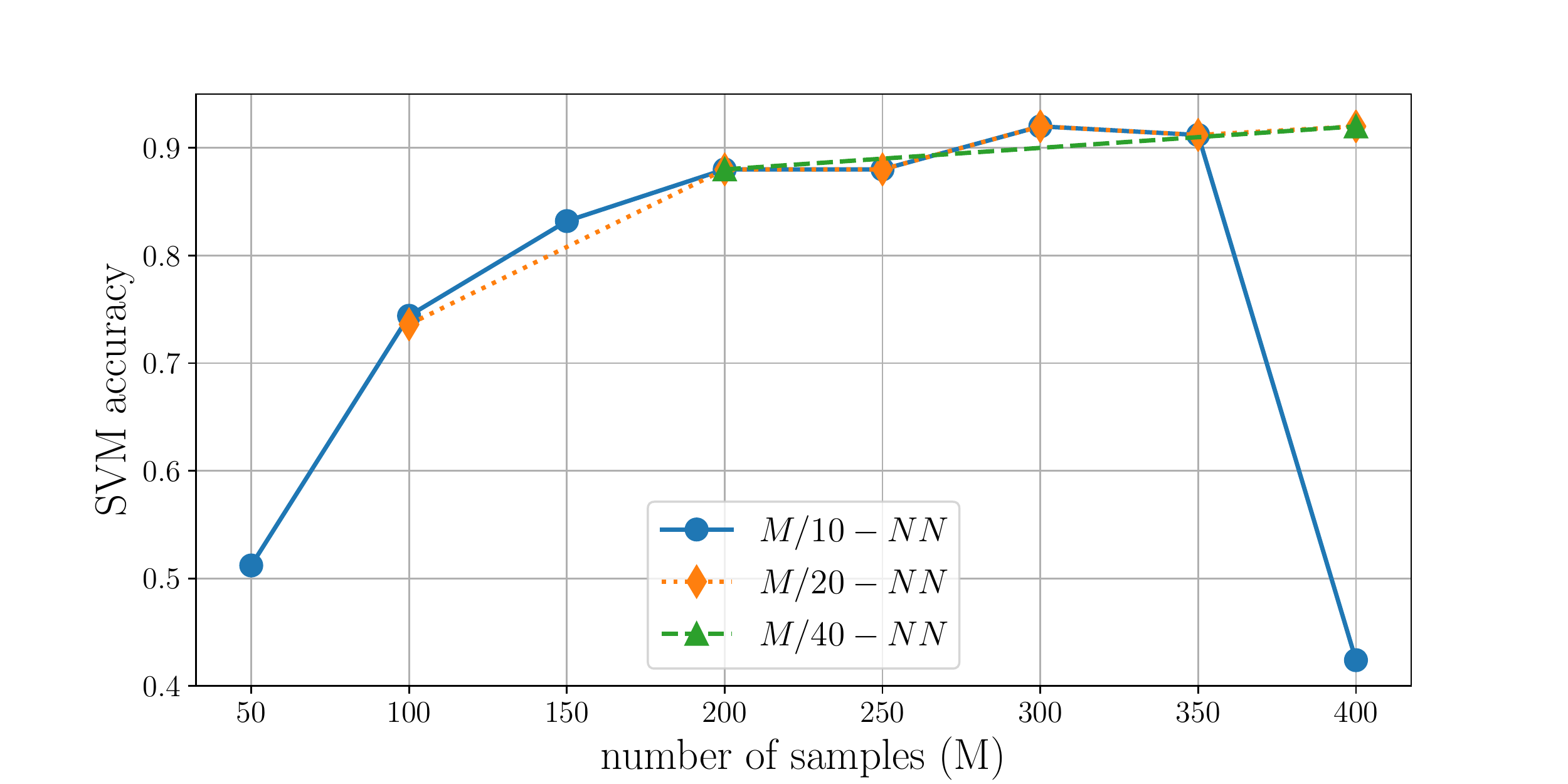}
\caption{Classification accuracy on a binary classification dataset, generated with $30$ features and $k=10$, for different values of number of samples and number of neighbors used for estimating the MI.}
\label{fig:binary_big}
    \label{fig:chainAll} 
  \end{minipage} 
\end{figure*}

\subsection{Real-World Data}
We further tested the proposed feature selection algorithm on several popular real world datasets, publicly available on the ASU feature selection website and the UCI repository~\cite{Dua:2017}.
In Table \ref{tab:real_data_res}, we report the classification accuracy on a test set after the feature selection procedure for different values of the threshold $\delta$.
We notice how the upper bound on the error is stricter in some examples and larger in others. In particular, the actual classification accuracy follows the theoretical error bound in cases where the dataset has a bigger number of samples and a number of features that is not too big, for example ORL. Conversely, if the number of features is too big in comparison to the number of samples, the error bound tends to be pessimistic and the actual accuracy is much bigger than the expected one (warpAR10P, ALLAML). Interestingly enough, the number of classes does not play a significant role.

\begin{table}
\caption{Real Data Results.}
\label{tab:real_data_res}
\centering
\begin{tabular}{cccccc}
\hline
Dataset & $\delta=0.05$ & $\delta=0.1$ & $\delta=0.25$ & $\delta=0.5$ & $\delta=1.0$ \\
\hline
ORL & 0.8  & 0.75  & 0.7 & 0.7375 & 0.7125 \\
warpAR10P & 0.97 & 0.98 &  0.98 &  0.98 & 0.98 \\
glass* & 0.99 & 0.99 & 0.99 & 0.99 & 0.99 \\
wine & 0.96 & 0.96 & 0.96 & 0.95 & 0.83 \\
ALLAML & 1.0 & 1.0 & 1.0 & 0.92 & 0.78 \\
\hline
\end{tabular}
\small *: no feature removed
\end{table}

\section{Discussion and Conclusion}
Conditional Mutual Information is an effective statistical tool to perform feature selection via filter methods. In this paper, we proposed a novel theoretical analysis showing that using CMI allows to control the ideal prediction error, assuming that the trained model has infinite capacity. This is a
rather new insight, as filter methods are typically employed when no assumptions are made on the 
underlying trained model. We proved that, when using linear models, the correlation
coefficient becomes a suitable criterion for ranking and selecting features. On the bases of our findings, we proposed a new stopping condition, that can be applied to both forward and backward feature selection, with theoretical guarantees on the prediction error.
The experimental evaluation showed that, compared against classical filter methods and stopping criteria, our approach, besides the theoretical foundation, is less sensitive to the choice of the threshold hyper-parameter and allows reaching state-of-the-art results.
%

\end{document}